\documentclass[11pt]{article}
\usepackage[sc]{mathpazo}
\usepackage{tikz}
\usepackage{amsfonts,latexsym,amsmath,stmaryrd,dsfont}
\usepackage{amsthm,amssymb,mathtools}
\usepackage{multido}
\usepackage{pst-plot}
\usepackage{graphicx}
\usepackage{fancybox}
\usepackage[normalem]{ulem}
\usepackage{amscd}
\usepackage{enumitem}
\usepackage[utf8]{inputenc}

\usepackage[all,cmtip]{xy}
\xyoption{frame}

\usepackage[nottoc,notlot,notlof]{tocbibind}

\makeatletter
\g@addto@macro\@floatboxreset\centering
\makeatother

\usepackage[margin=1in]{geometry}
\newtheorem{theorem}{Theorem}
\newtheorem{lemma}[theorem]{Lemma}
\newtheorem{corollary}[theorem]{Corollary}

\newtheorem{prop}[theorem]{Proposition}

\newtheorem{obs}[theorem]{Observation}

\newcommand{\suppress}[1]{}

\def\rset{\mathbb R}
\def\qed{\mbox{ }~\hfill~$\Box$}

\def\zset{\mathbb Z}
\def\cset{\mathbb C}

\def\qset{\mathbb Q}

\def\ep{\varepsilon}

\def\cU{{\mathcal U}}

\def\cX{{\mathcal X}}

\DeclareMathOperator{\sign}{sign}

\DeclareMathOperator{\Han}{Han}

\newcommand{\be}{\begin{equation}}
\newcommand{\ee}{\end{equation}}
\newcommand{\bea}{\begin{eqnarray}}
\newcommand{\eea}{\end{eqnarray}}
\newcommand{\bean}{\begin{eqnarray*}}
\newcommand{\eean}{\end{eqnarray*}}
\newrgbcolor{darkgreen}{0 0.6 0}

\newcommand{\mun}{\mathbold \mu}
\newcommand{\pn}{\mathbold p}

\newcommand{\an}{a}

\def\Floor#1{\left\lfloor #1 \right\rfloor}

\def\Paren#1{\left( #1 \right)}		

\newcommand{\mathcolorbox}[1]{\colorbox{gray!30}{$\displaystyle #1$}}

\usepackage[ruled,vlined]{algorithm2e}
\usepackage{dutchcal}
\usepackage{url}
\usepackage[dvips]{epsfig}
\usepackage{hyperref}
\hypersetup{colorlinks = true, linkcolor =blue, anchorcolor = red, citecolor = blue, urlcolor = blue}
\hypersetup{linktocpage}

\makeatletter
\newcommand{\mylabel}[2]{#2\def\@currentlabel{#2}\label{#1}}
\makeatother

\newcommand{\rtl}{s}
\newcommand{\fv}{x}
\newcommand{\sv}{y}

\author{Spencer L. Gordon, Manav Kant, Eric Ma, Leonard J. Schulman and Andrei Staicu\thanks{Research supported in part by NSF grants CCF-1909972 and CCF-2321079. The authors would like to acknowledge the Caltech SURF program and the Larson, Mike Stefanko, and SURF Board endowments. The authors are with the Division of Engineering and Applied Science, California Institute of Technology, Pasadena CA 91125 USA (emails: {\tt slgordon, mkant, ema, astaicu, schulman@caltech.edu}).}}

\title{Identifiability of Product of Experts Models}
\begin{document}

\maketitle

\begin{abstract} 
Product of experts (PoE) are layered networks in which the value at each node is an AND (or product) of the values (possibly negated) at its inputs. These were introduced as a neural network architecture that can efficiently learn to generate high-dimensional data which satisfy many low-dimensional constraints---thereby allowing each individual expert to perform a simple task. PoEs have found a variety of applications in learning. 

We study the problem of identifiability of a product of experts model having a layer of binary latent variables, and a layer of binary observables that are iid conditional on the latents. The previous best upper bound on the number of observables needed to identify the model was exponential in the number of parameters. We show: (a) When the latents are uniformly distributed, the model is identifiable with a number of observables equal to the number of parameters (and hence best possible). (b) In the more general case of arbitrarily distributed latents, the model is identifiable for a  number of observables that is still linear in the number of parameters (and within a factor of two of best-possible). The proofs rely on root interlacing phenomena for some special three-term recurrences.
\end{abstract} 


\section{Introduction}
\paragraph{Product of experts models.} 
In modeling complex, high-dimensional data, it is often necessary to combine various simple distributions to produce a more expressive distribution. One way of doing this is the mixture model, or weighted sum of distributions. Alone, however, this still requires quite expressive components, which is a hindrance for modeling data in a high-dimensional space. Product of experts (PoE) were introduced in the neural networks literature as an antidote to this problem: the distribution is factorized into a set of independent lower-dimensional distributions~\cite{Hinton99,Hinton02}. Equivalently, the overall distribution is an AND over the factor distributions (which may themselves be mixture models).

PoEs have recently been applied to solve diverse problems requiring the generation of data that simultaneously satisfy numerous sets of constraints. For example, the PoE-GAN framework has advanced the state of the art in multimodal conditional image synthesis, generating images conditioned on all or some subset of text, sketch, and segmentation inputs~\cite{HuangMWL22}. In the field of language generation, a PoE with two factors, a pre-trained language model and a combination of a toxicity expert with an anti-expert, was used to steer a language model away from offensive outputs~\cite{LiuSLSBSC21}.

However, fundamental questions about even the simplest PoE models remain unresolved. In this paper, we study a PoE in which each observable random variable (rv) $X$ depends upon statistically independent latent rv's $\cU=(U_1,\dotsc, U_{\ell})$. The \emph{model} here is the prior distributions on the $U_j$, and the statistics are the likelihoods $\Pr(X=x|\cU)$. 

We investigate the well-studied class of instances in which the $X_i$ and the $U_j$ are binary and $\Pr(X=1|\cU)$ can be expressed as a product over the $U_j$'s, namely, for some coefficients $\alpha$,
\be \Pr(X=1\mid \cU=u) = \prod_{j=1}^{\ell} \alpha_{j,u_j} \label{prod:struct} \ee
It is worth noting the latent symmetry present in this class, in that the distribution of $X$ conditional on $\mathcal U$ is invariant to permutation of the $\mathcal U_j$'s.

The question with which we are concerned is: what is the minimum $n$ needed so that instances of this class can be identified from the statistics of $n$ independent samples $X_1,\ldots,X_n$? It is necessary to sample at least as many observable variables as there are parameters. However, the previous best upper bound on this value was \emph{exponential} in the number of parameters. 

We show that identifiability holds in the least possible number of observables, in the case of (identical) uniform priors on the $U_j$'s; and in twice the least possible in the case of general priors. This resolves the previous exponential gap that existed for this problem.

\paragraph{Algebraic mappings.} 
Note that the model delineated above corresponds to a directed graphical model, or Bayesian network, 
 with all edges directed from latent toward observable vertices. The independence of the $U_j$'s in the prior is implicit in that these vertices are sources in this graph: see Fig.~\ref{fig:dir}.
 The representation~\eqref{prod:struct}
 specifies an algebraic mapping
from model parameters (namely, the $\alpha$'s and the prior distributions on the $U_j$)
to a probability distribution on $\cX$. 
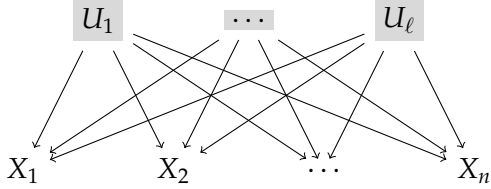
\begin{figure} 

\begin{tikzpicture}
  \node (U1) at (0,0) {$\mathcolorbox{U_1}$};
  \node (U2) at (2, 0) {$\mathcolorbox{\ldots}$};
  \node (U3) at (4, 0) {$\mathcolorbox{U_\ell}$};
  
  \node (X1) at (-1, -2) {$X_1$};
  \node (X2) at (1, -2) {$X_2$};
  \node (X3) at (3, -2) {$\ldots$};
  \node (X4) at (5, -2) {$X_n$};
  
  \foreach \i in {1,2,3} {
    \foreach \j in {1,2,3,4} {
      \draw[->] (U\i) -- (X\j);
    }
  }\end{tikzpicture}\caption{Directed graphical model (latent variables shaded)} \label{fig:dir} \end{figure}

In Section~\ref{sec:un} the prior on $U$ is assumed uniform, so the mapping we are concerned with is from the $\alpha$'s to the distribution of $\cX$. 
Section~\ref{sec:gen} treats the general case of arbitrary priors on the $U_j$.

If the dimension of the image of an algebraic mapping matches the dimension of the parameter space (minus any continuous latent symmetries), it follows that the image of any regular point has only finitely many pre-images. In this case we say the model is locally identifiable. If, further, at such points the pre-image consists only of points which are in a common orbit of the latent symmetry group, then we say that the model is fully identifiable. The question of local identifiability along with the related question of whether the model generates all possible distributions on $\cX$, are the subjects of much of the extensive literature about this model and its undirected-graph variant (see below) in the neural networks and algebraic statistics literatures~\cite{hintonOT06,drtonSS07,LeRouxB08,CuetoMS10,longS10,Hinton10,MartensCPZ13,MontufarM15,MontufarM17,Montufar18,SeigalM18}. We encountered this question in the context of causal identification and the classical moment problem~\cite{gordon2020sparse,gordon2021source,gmrs23}.

The only prior upper bound for the for the number of observables required for model identification, followed by treating $\cU$ as a single ``lumped'' rv, taking on $2^\ell$ possible values. Then there is a longstanding result~\cite{Blischke64} implying an upper bound on $n$ of $2^{\ell+1}-1$. This is exponentially larger than the number of parameters of the model.

We close this exponential gap, both for uniform and for general priors. We introduce a general method for showing local identifiability of these latent symmetric models. Our main results are: \newline \newline
\emph{Theorem~\ref{thm:independentinfluence}: In the uniform-prior case, the model can be identified with $n=\ell$ observables, which is best possible, as it matches the number of degrees of freedom of the model (after quotienting-out symmetries). \newline \newline
Theorem~\ref{thm:ind-inf-genl}: 
 In the general-prior case, the model can be identified with $n\leq 4\ell+1$ observables, which is within a factor of two of best possible.}  

These theorems are logically incomparable. Their proofs share a basic structure, but the second involves some surprising ingredients that are not foreshadowed in the first.

This result opens up the possibility of a fast algorithm for the identification of PoEs belonging to the given class, since the running time of any such algorithm is bounded from below by the number of moments required for identification. We envision a research program culminating in a fast algorithm for robust identification of product of experts models.

\paragraph{Related work: undirected graphical models.}
The \emph{conditional distributions} on $X$ that occur in our model~\eqref{prod:struct} agree with those of 
 Restricted Boltzmann Machines (RBM)~\cite{ackleyHS85}
(a kind of Markov random field model), and particularly the ``harmonium''~\cite{smolensky86,freundH91} special case. An RBM is an undirected graphical model comprising one layer of latent random variables $\cU$, one layer of observable random variables $\cX$, and satisfying that the $X_i$ are independent conditional on $\cU$.
An RBM is often written in the following form:
\be \Pr((\cX,\cU)=(x,u)) = \frac{1}{Z} \exp(-x W u^\dagger - x b^\dagger - c u^\dagger) \label{RBM1} \ee where
$Z=\sum_{x,u} \exp(-x W u^\dagger - x b^\dagger - c u^\dagger)$. The dependence of $X$ on $U$ can be expressed:
\be \Pr(x|u)=\exp(-x W u^\dagger - x b^\dagger - c u^\dagger-d) \label{RBM2} \ee
where $d=-\sum_{i,j} \log (1-\alpha_{i,j,0}) ; \; 
b_i =-\sum_{j}\log \frac{\alpha_{i,j,0}}{1-\alpha_{i,j,0}} ; \;
c_j=-\sum_i \log \frac{1-\alpha_{i,j,1}}{1-\alpha_{i,j,0}} ; \;
W_{ij}=-\log\frac{\alpha_{i,j,1}(1-\alpha_{i,j,0})}{\alpha_{i,j,0}(1-\alpha_{i,j,1})}.$ This is referred to as an undirected graphical model because each coefficient $W_{ij}$ is regarded as an energy associated with the unordered pair of sites $\{i,j\}$ ($i$ observable, $j$ latent). The conditional distributions $\Pr(\cX|\cU)$ in~\eqref{RBM2} have 
 the same form as in our model~\eqref{prod:struct} (or its more general version in which the $X_i$ are only conditionally independent)---but~\eqref{RBM2} does not allow imposition of 
 a chosen product distribution on $\cU$ as the prior, and in fact, generally the $U_j$ will not be independent in the distribution~\eqref{RBM1}. 

The conditional distributions of RBMs enable expressive statistical models with relatively few parameters. For this reason and because of the connection to layered networks, RBMs have been extensively studied in the neural networks and algebraic statistics literatures (citations above). An interesting recent (and independent) contribution in this literature concerns identifiability~\cite{OV23}, but there are no obvious implications in either direction between the works.
(The main thing to note is that it relies for identifiability on the Kruskal condition; that condition does gives an upper bound 
on the number of observables required but if one works out the bound, one sees
that it cannot be less than
 $2^{\ell+1}-1$. But the slightly better bound of $2^\ell$ was the exponentially-large bound that we set out to rectify. Second, one should note that that work addresses a somewhat more general class of models, but then relies upon a general-position assumption about the input, an assumption we do not make, and that is also not valid in our setting of conditionally-iid $X_i$.)

\section{Identifying models with uniform priors} \label{sec:un} \subsection{Preliminaries}
\paragraph{The model.}

In this section we treat the case that the prior on each $U_j$ is uniform. 
Then the model has $2\ell$ parameters $\alpha_{jb}$ for $j=1,\dotsc,\ell$ and $b = 0,1$;
and~\eqref{prod:struct} yields that
\begin{align} \begin{split}\Pr(X=1)&=\mun_1 \quad \quad \text{ where } \\
 \mun_1 & \coloneqq \prod_{j=1}^\ell \frac{\alpha_{j0} + \alpha_{j1}}{2} \end{split} \label{mun1} \end{align}
and more generally
\begin{align} \begin{split}
\Pr(X_1=\ldots=X_t=1) &= \mun_t \quad \quad \text{ where } \\
 \mun_t & \coloneqq \prod_{j=1}^\ell \frac{\alpha_{j0}^t + \alpha_{j1}^t}{2}.
\end{split} \label{munj} \end{align}

\paragraph{Hadamard products of Hankel matrices.} It was observed in~\cite{CuetoMS10} that the RBM model under consideration there, was a Hadamard product of simpler models, each having only a single latent variable. A similar phenomenon occurs in the directed model we study.

For each $j \in [\ell]$, define $\Han(n,j)$ to be the $(n+1)\times(n+1)$ matrix
with entries \[ \Han(n,j)_{ab}=(\alpha_{j0}^{a+b} + \alpha_{j1}^{a+b})/2 \quad \quad (0\leq a,b \leq n). \] This is a Hankel matrix of rank $2$. The Hadamard (or entrywise) product of these matrices, ranging over $1\leq j \leq \ell$, is
\[ H(n)_{ab}= \mun_{a+b} \quad \quad (0\leq a,b \leq n). \]
If we had access to each individual Hankel matrix $\Han(2,j)$, we could apply the classical method of Prony or similar methods~\cite{Prony1795,gordon2020sparse,KKMMR18} to determine $\alpha_{j0}$ and 
$\alpha_{j1}$; this method works (for an arbitrary prior on $U_j$)
provided the Hankel matrix is rank-deficient.

However, what we have access to is only $H(n)$ and not the individual $\Han(n,j)$'s. One may consider $H(n)$ itself as a Hankel matrix, by ignoring the product structure on the $U_j$'s and regarding $U$ as a single latent variable with range $[2^\ell]$; but then in order to have rank deficiency, one requires the very large Hadamard matrix $H(2^\ell)$, and consequently, the Prony method is only applicable to identifying the model if we obtain $n=2^\ell$ observables $X_i$. This is exponentially larger than the number of degrees of freedom of the model, which as we see from the model \eqref{munj} is merely linear in $\ell$. As noted earlier, this exponential gap was the motivation for our investigation.
We show that the linear answer is correct: the model is locally identifiable as soon as $n$ matches the number of degrees of freedom of the model (after quotienting out a continuous symmetry which we shortly describe).

\paragraph{Symmetries of the model.} The model has both discrete and continuous symmetries. The moments $\mun_t$ are invariant to:
\begin{enumerate}
\item \emph{Discrete symmetries.} The wreath product $S_2 \wr S_\ell$ (a.k.a.\ hyperoctahedral group):
\begin{enumerate}[label=\textnormal{(\alph*)}] 
\item For any $j$, exchange $\alpha_{j0}$ and $\alpha_{j1}$.
\item Exchange any $j$ and $j'$. That is, for $j\neq j'$, exchange
$\alpha_{j0}$ with $\alpha_{j'0}$, and $\alpha_{j1}$ with $\alpha_{j'1}$. 
\end{enumerate}
\item
\emph{Continuous symmetries.} For any $j\neq j'$ and  $\lambda > 0$, the gauge transformation
\begin{align} \begin{cases}
    (\alpha_{j0},\alpha_{j1}) &\mapsto (\lambda\alpha_{j0},\lambda\alpha_{j1}),  \\
    (\alpha_{j'0},\alpha_{j'1}) &\mapsto (\lambda^{-1}\alpha_{j'0},\lambda^{-1}\alpha_{j'1})
    \end{cases} \label{gauge}
\end{align} 
\end{enumerate}

The model can of course be identified only up to these symmetries. We can therefore, w.l.o.g., take advantage of the gauge symmetries to scale the parameters $\alpha_{jb}$ so that for all $j\in [\ell]$,
\be \alpha_{j0}+\alpha_{j1}=2\mun_1^{1/\ell} =: \gamma \label{gamma-known} \ee
 We see now that the model has only $\ell$ genuine continuous degrees of freedom. 

If $\mun_1=0$ the model is trivial, so in the sequel we assume $\gamma>0$.

To solve for $\alpha_{j0}$ and $\alpha_{j1}$, up to the symmetry between them, it suffices, in view of~\eqref{gamma-known}, to solve for \[ {\an}_j \coloneqq \alpha_{j0}\alpha_{j1}. \]

This transformation results in a family of polynomials 
$q_t$ in ${\an}_1,\ldots,{\an}_\ell$ such that $\mun_t = q_t({\an}_1,\ldots,{\an}_\ell)$; fixing any $m$, the mapping $(q_1,\ldots,q_m$) of 
$\{{\an}_1,\ldots,{\an}_\ell \}$ to $(\mun_1,\ldots,\mun_m)$ 
carries $\rset^\ell$ into a variety of dimension at most $\ell$. If the dimension is $\ell$, then for all but a set of measure $0$, specifically at all regular points of the mapping, the image has finitely many pre-images.

Our main result in this section is:
\begin{theorem} The mapping $\{\an_1,\ldots,\an_\ell\} \mapsto (\mun_1,\dotsc,\mun_{\ell})$ is a.e.\ locally identifiable. \label{thm:independentinfluence}
\end{theorem}

\subsection{The polynomial sequence} \label{sec:un-p-s}
We require a more detailed understanding of the probabilities $\mu_t$. Observe that 
\[\alpha_{j0}^2 + \alpha_{j1}^2 = (\alpha_{j0}+\alpha_{j1})^2 - 2\alpha_{j0}\alpha_{j1} = \gamma^2 - 2\an_j \] so that 
\[ \mun_2 = \prod_{j=1}^{\ell} \frac{\gamma^2-2\an_j}{2}. \]

Similarly, $
    \alpha_{j0}^3 + \alpha_{j1}^3 = (\alpha_{j0} +\alpha_{j1})^3 - 3(\alpha_{j0}^2\alpha_{j1}+\alpha_{j0}\alpha_{j1}^2)
    = \gamma^3 - 3\gamma \an_j$, so that $\mun_3$ has an analogous expression. This continues. In the remainder of this Section the subscript `$j$' is suppressed.

\begin{lemma}[Three-term recurrence] Letting $\an = \alpha_0\alpha_1$,
$(\alpha_0^m + \alpha_1^m)/2$ 
can be written as a polynomial $\pn_m(\an)$ 
satisfying the three-term recurrence \label{prop-recur}
\begin{align}
\pn_0 (\an) &= 1, \nonumber \\
\pn_1(\an) &= \gamma/2, \nonumber \\
\pn_m(\an) &= \gamma \pn_{m-1} - \an \pn_{m-2},\quad m\geq 2. \label{recur3}
\end{align} 
\end{lemma}
\begin{proof} This is the Newton identity relating power and elementary symmetric functions, specialized to the two-variable case; $\gamma$ is the first, and $\an$ is the second, elementary symmetric function of $\alpha_0$ and $\alpha_1$.
\end{proof}

The recurrence~\eqref{recur3} resembles the 
familiar recurrence of orthogonal polynomials, but differs significantly in that the variable (here $\an$) multiplies $\pn_{m-2}$ rather than $\pn_{m-1}$. In particular the polynomials $\pn_m$ are not of incrementing degree in $\an$.

(We note however that at the particular value $\an=1/4$, these polynomials 
as functions of $\gamma$ are rescalings of the Chebychev polynomials of the first kind.) 
 
\begin{obs} The polynomials $\pn_m$ defined above satisfy the following: 
\begin{enumerate}
\item The degree of $\pn_m(\an)$ is $\Floor{m/2}$.
\item $\pn_m(0) = \gamma^m/2$ for $m\geq 1$.
\item The leading term of $\pn_m(\an)$ has a negative sign if $\Floor{m/2}$ is odd, and a positive sign if $\Floor{m/2}$ is even.
\end{enumerate} \label{obsp}
\end{obs}
\begin{proof} By induction over $m$. 
\end{proof}

\begin{prop}[Interlacing] $\pn_0$ and $\pn_1$ are positive constants. For any $m \geq 2$:
\begin{enumerate} 
\item The roots of $\pn_m$ are simple and contained in the interval $(0,\infty)$; denote them
$\beta_{m,1} < \dotsc< \beta_{m,{\Floor{m/2}}}$. 
\item $0 < \beta_{m,1} < \beta_{m-1,1}$ and $\beta_{m-1,i-1} < \beta_{m,i} < \beta_{m-1,i}$ for $i=2,\dotsc,\Floor{m/2}$. If $\pn_m$ has degree greater than $\pn_{m-1}$ then $\beta_{m-1,{\Floor{(m-1)/2}}} < \beta_{m,{\Floor{m/2}}}$.

(For $m=2$ this requires the convention $\beta_{1,0}=0,\beta_{1,1}=\infty$.)
\end{enumerate} \label{prop:characterizing univariate polynomials}
\end{prop}
\begin{proof}
We induct on $m$. The Proposition holds for $m=0,1$. Now $\pn_2 = -\an+\gamma^2/2$ which has a root at $\gamma^2/2$ and $\pn_3 = \gamma(-\an+\gamma^2/2)-\an\gamma/2 
= -3\gamma \an/2 + \gamma^3/2$ which has a root at $\gamma^2/3$.

Now fix any $m\geq 2$. Let $d = \deg(\pn_{m-2})$, so $d+1=\deg(\pn_m)$. 
By the inductive hypothesis, $0<\beta_{m-1,1}<\beta_{m-2,1}<\beta_{m-1,2}<\dotsm<\beta_{m-1,d}<\beta_{m-2,d}$. 
If $\deg(\pn_{m-1}) = d+1 > \deg(\pn_m)$, then there is an additional root $\beta_{m-1,d+1}$ of $\pn_{m-1}$ with $\beta_{m-2,d} < \beta_{m-1,d+1}$.
Since we've accounted for every root of $\pn_{m-1}$ and $\pn_{m-2}$,
 the value of $\pn_{m-1}$ must alternate between strictly positive and strictly negative on the sequence of open intervals $(-\infty,\beta_{m-1,1})$, $(\beta_{m-1,1},\beta_{m-1,2})$, $(\beta_{m-1,2},\beta_{m-1,3}),\dotsc,$ $(\beta_{m-1,{\Floor{(m-1)/2}}},\infty)$, 
and $\pn_{m-2}$ alternates in sign on the intervals $(-\infty,\beta_{m-2,1})$, $(\beta_{m-2,1},\beta_{m-2,2}),\dotsc,$ $(\beta_{m-2,d},\infty)$.
Now we compute 
\begin{align*}
\pn_m(\beta_{m-1,1}) &= \gamma \pn_{m-1}(\beta_{m-1,1}) - \beta_{m-1,1} \pn_{m-2}(\beta_{m-1,1}) &&= -\beta_{m-1,1}\pn_{m-2}(\beta_{m-1,1}) &< 0;\\ 
\pn_m(\beta_{m-2,1}) &= \gamma \pn_{m-1}(\beta_{m-2,1}) - \beta_{m-2,1} \pn_{m-2}(\beta_{m-2,1}) &&= \gamma \pn_{m-1}(\beta_{m-2,1}) &< 0.
\end{align*} By Observation~\ref{obsp}, $\pn_m(0) = \gamma^m/2 > 0$ so there must be a root of $\pn_m$ in $(0,\beta_{m-1,1})$.

For $1 < i < d$, $\beta_{m-1,i} \in (\beta_{m-2,i-1},\beta_{m-2,i})$ and $\beta_{m-2,i} \in (\beta_{m-1,i},\beta_{m-1,i+1})$.
Moreover, $\pn_m(\beta_{m-1,i}) = -\beta_{m-1,i} \pn_{m-2}(\beta_{m-1,i})$ and $\pn_m(\beta_{m-2,i}) = \gamma \pn_{m-1}(\beta_{m-2,i})$, so $\sign(\pn_m(\beta_{m-1,i})) = \sign(\pn_m(\beta_{m-2,i})) = - \sign(\pn_m(\beta_{m-1,i+1}))$. We conclude that there is a root of $\pn_m$ in the interval $(\beta_{m-2,i},\beta_{m-1,i+1})$ for $1 < i < d$. 

We've shown that there are roots of $\pn_m$ in each of the intervals 
$(0,\beta_{m-1,1})$, $(\beta_{m-2,1}$,$\beta_{m-1,2})$,
$(\beta_{m-2,2},\beta_{m-1,3}),\dotsc,$ $(\beta_{m-2,d-1}$, 
$\beta_{m-1,d})$. 
If $\deg(\pn_{m-1}) = d+1$, then by the same logic there is also a root in $(\beta_{m-2,d},\beta_{m-1,d+1})$ and the proof is complete. If $\deg(\pn_{m-1}) = d$, then the leading term of $\pn_m$ has a different sign than the leading terms of $\pn_{m-1}$ and $\pn_{m-2}$. Since $\sign(\pn_m(\beta_{m-2,d})) = \sign(\pn_{m-1}(\beta_{m-2,d}))$ and $\beta_{m-2,d}$ is greater than all the roots of $\pn_{m-1}$ it must be the case that $\sign(\pn_{m-1}(y)) = \sign(\pn_{m-1}(\beta_{m-2,d}))$ for all $y \in [\beta_{m-2,d},\infty)$. 
But $\lim_{y\to \infty} \pn_m(y) =-\lim_{y\to \infty}\pn_{m-1}(y)$, so there must be a root of $\pn_m$ in $(\beta_{m-2,d},\infty)$. We've thus accounted for all $d+1$ roots of $\pn_m$.
\end{proof} 

\subsection{Identifiability of latent symmetric models: the method} \label{uni-method}
We now describe the algebraic tool which enables the proof of Theorem~\ref{thm:independentinfluence}.
Consider a sequence of $\ell$ univariate polynomials $\pn_1,\dotsc,\pn_{\ell}$. (We will eventually substitute the $\pn$'s of the previous Section.) Let $y_1,\dotsc,y_{\ell}$ be indeterminates, and $y=(y_1,\dotsc,y_{\ell})$.
Construct symmetric polynomials in the $y_j$ by taking products as follows: 
\[ q_m(y) = \prod_{j=1}^{\ell} \pn_m(y_j). \]

\begin{prop} 
Suppose that for every $m\in [\ell]$, $\pn_m$ has a root $\eta_m$ 
that is simple and is not a root of $\pn_1,\dotsc,\pn_{m-1}$. Then the mapping \[ (y_1,\dotsc,y_{\ell}) \mapsto (q_1(y),\dotsc,q_{\ell}(y))\] is locally identifiable. \label{prop:identifiability of symmetric polynomials}
\end{prop}
A stronger version of this Proposition is Theorem~\ref{blowup} in Appendix~\ref{Jac-cond}. We prove the Proposition here however as it is all we need for Theorem~\ref{thm:independentinfluence}, and the proof is slightly more direct. 
\begin{proof} 
It suffices to show that there is a point at which the Jacobian of the mapping is nonsingular.
In what follows for a rational function $g$ let $M(g,\eta_j)$ denote the multiplicity of $\eta_j$ as a root of $g$; 
if $\eta_j$ is a pole of $g$ then $-M(g,\eta_j)$ is the order of the pole. 

By assumption, $M(\pn_i,\eta_j)=0$ for $j > i$ and $M(\pn_j,\eta_j)=1$ for all $j$.

We now construct a sequence of rational functions $r_1,\dotsc,r_{\ell}$ 
satisfying (with $\delta_{ij}=$ Kronecker delta):
\[ M(r_i,\eta_j) = \delta_{ij}. \] 
First, we set $r_1 = \pn_1$, since $M(\pn_1,\eta_j)=0$ for all $j > 1$. 

Inductively we construct $r_i$ for $i\geq 2$ as follows:
\be r_i = \pn_i\prod_{i'=1}^{i-1}r_{i'}^{-M(\pn_i,\eta_{i'})} .\label{ge-r} \ee
By construction, $M(r_i,\eta_j) = 0$ for $j< i$ and $M(r_i,\eta_i)= 1$. Moreover, $M(r_i,\eta_j)= 0$ for $j > i$ since 
\[M(\pn_i,\eta_j)= M(r_1,\eta_j)= \dotsm = M(r_{i-1},\eta_j). \]

Define $s_i(y) = \prod_{j=1}^\ell r_i(y_j)$ for $i=1,\dotsc,\ell$ so that $s_i$ is the product of $r_i$ evaluated at each indeterminate, just as $q_i$ is the product of $\pn_i$ evaluated at each indeterminate. In fact, we have 
\[s_i = q_i \prod_{i'=1}^{i-1}s_{i'}^{-M(\pn_i, \eta_{i'})}. \]

Let $Q$ and $S$ be the following mappings:
\[ (y_1,\dotsc,y_{\ell}) \xmapsto{Q} (q_1,\dotsc,q_{\ell}) \xmapsto{S} (s_1,\dotsc,s_{\ell}).\]
Consider the Jacobian of $S\circ Q$, evaluated at the point $\eta=(\eta_1,\dotsc, \eta_{\ell})$.
By construction 
\be \frac{\partial s_i}{\partial y_j}(\eta) = \Paren{\prod_{j'\neq j} r_i(\eta_{j'})} r'_i(\eta_j). \label{pspe} \ee
Now \[ \prod_{j'\neq j} r_i(\eta_{j'}) \neq 0 \iff i=j \] since $r_i(\eta_i) = 0$ and $M(r_i,\eta_j) = 0$ for any $j\neq i$. Moreover, $r'_i(\eta_i) \neq 0$ since $\eta_i$ is a simple root of $r_i$, so we conclude that $\frac{\partial s_i}{\partial y_j}(\eta) \neq 0 \iff i = j$. Thus, the Jacobian is a diagonal matrix with non-zero diagonal entries and is therefore invertible. 
The Proposition follows.
\end{proof}
\begin{proof}[Proof of Theorem~\ref{thm:independentinfluence}] 
Apply Prop.~\ref{prop:identifiability of symmetric polynomials} with the polynomials $\pn$ being as defined in Lemma~\ref{prop-recur}, and 
 the root $\eta_m$ of $\pn_m$ being the point $\beta_{m,1}$ provided by Prop.~\ref{prop:characterizing univariate polynomials}. \end{proof}

\section{Identifying models with general priors} \label{sec:gen}
\subsection{Preliminaries} We now treat the more general setting where $U_j$ have arbitrary priors, specified by the parameters
\[ \pi_j=\Pr(U_j=1).\]
It will be convenient to also use the notation $\pi_j(u)$, with $\pi_j=\pi_j(1)=1-\pi_j(0)$. 

Because the $U_j$'s are not uniformly sampled, the moments will take on a different form than before (we call the $n$th moment $q_n$):
\begin{align}
q_n \coloneqq \Pr(X_1 = \cdots = X_n = 1) 
&= \sum_{u \in \{0, 1\}^\ell} \prod_{j = 1}^\ell \pi_j(u_j) \alpha_{j, u_j}^n 
= \prod_{j = 1}^\ell \left(\pi_j(0)\alpha_{j, 0}^n + \pi_j(1)\alpha_{j, 1}^n \right) \nonumber\\ 
&= \prod_{j = 1}^\ell \left((1 - \pi_j)\alpha_{j, 0}^n + \pi_j\alpha_{j, 1}^n\right)  \nonumber\\ 
&= \prod_{j = 1}^\ell r_n(\alpha_{j,0},\alpha_{j,1},\pi_j) \label{qrfactor}
\end{align}
where we define $r_n(\alpha_0,\alpha_1,\pi)=\pi \alpha_1^n + (1-\pi)\alpha_0^n$. We may view $r_n(\alpha_{j,0},\alpha_{j,1},\pi_j)$ as the contribution of $U_j$ to the moment $q_n$.

\paragraph{Symmetries of the model.} The model has discrete and continuous symmetries as before. The moments $q_t$ are invariant to:
\begin{enumerate}
\item \emph{Discrete symmetries (hyperoctahedral).}
\begin{enumerate}[label=\textnormal{(\alph*)}] 
\item For any $j$, exchange $\alpha_{j0}$ and $\alpha_{j1}$, and replace $\pi_j$ with $1-\pi_j$.
\item Exchange any $j$ and $j'$. That is, for $j\neq j'$, exchange
$\alpha_{j0}$ with $\alpha_{j'0}$, $\alpha_{j1}$ with $\alpha_{j'1}$, and $\pi_j$ with $\pi_{j'}$. 
\end{enumerate}
\item
\emph{Continuous symmetries.} For any $j\neq j'$ and  $\lambda > 0$, the gauge transformation
\begin{align*} \begin{cases}
    (\alpha_{j0},\alpha_{j1}) &\mapsto (\lambda\alpha_{j0},\lambda\alpha_{j1}),  \\
    (\alpha_{j'0},\alpha_{j'1}) &\mapsto (\lambda^{-1}\alpha_{j'0},\lambda^{-1}\alpha_{j'1})
    \end{cases}.
\end{align*} 
\end{enumerate}
Of course, the model may only be identified up to these symmetries. Therefore, as before, we use the gauge symmetries to scale our parameters such that for all $j\in[\ell]$, letting $\gamma \coloneqq \Pr(X_1=1)^{1/\ell}$, we have that $r_1(\alpha_{j,0},\alpha_{j,1},\pi_j)=\gamma$. As in Sec.~\ref{sec:un}, the model is trivial if $\gamma=0$, so we assume throughout that $\gamma\neq 0$.

\subsection{Polynomial sequences}
It will be convenient to make the change of variables $\sigma_j = 2\pi_j - 1$.
Due to the factorization~\eqref{qrfactor}, we can while studying the polynomials $r$, 
focus on an arbitrary $j$, and drop the indices $j$ until we return to treating the polynomials $q$. We can now write
\begin{align*}
    r_n = \pi \alpha_1^n + (1 - \pi) \alpha_0^n = \frac{\alpha_1^n + \alpha_0^n}{2} + \sigma \frac{\alpha_1^n - \alpha_0^n}{2}.
\end{align*}
Let $d = (\alpha_1 - \alpha_0)/{2}$.
\begin{prop}
$r_n = 2(\gamma - \sigma d) r_{n - 1} - [(\gamma - \sigma d)^2 - d^2] r_{n - 2}$ for all $n > 1$.
\end{prop}
\begin{proof}
Observe that $r_0 = 1$ and recall that we have set $r_1=\gamma$. Note that $2(\gamma - \sigma d) = \alpha_1 + \alpha_0$ and $(\gamma - \sigma d)^2 - d^2 = \alpha_1\alpha_0$. So
\begin{align*}
2(\gamma - \sigma d) r_{n - 1} - [(\gamma - \sigma d)^2 - d^2] r_{n - 2} &= 
 (\alpha_1 + \alpha_0) r_{n - 1} - \alpha_0\alpha_1 r_{n - 2}\\
&= (\alpha_1+\alpha_0) \left(\frac{\alpha_1^{n-1} + \alpha_0^{n-1}}{2} + \sigma \frac{\alpha_1^{n-1} - \alpha_0^{n-1}}{2}\right) \\ 
& \quad     -\alpha_0\alpha_1 \left(\frac{\alpha_1^{n-2} + \alpha_0^{n-2}}{2} + \sigma \frac{\alpha_1^{n-2} - \alpha_0^{n-2}}{2}\right) \\
       &= \frac{\alpha_1^n + \alpha_0^n}{2} + \sigma \frac{\alpha_1^n - \alpha_0^n}{2} = r_n
\end{align*}
\end{proof}

We make a final change of variables to replace $(\sigma,d)$ by $(\fv,\sv)$: $\fv = \gamma - \sigma d$ and $\sv = d^2$. (This is invertible after quotienting by the hyperoctahedral symmetry of the model.) This yields the following three-term recurrence:
\begin{align*}
    r_0(\fv, \sv) &= 1\\
    r_1(\fv, \sv) &= \gamma\\
    r_n(\fv, \sv) &= 2\fv r_{n-1} - (\fv^2 - \sv)r_{n-2}.
\end{align*}
It is helpful to define the following family of polynomials $(p_n)$, which are very closely related to the polynomials $(r_n)$. In particular, we will see in the following proposition that they are the ``coefficient polynomials'' of the $r_n$.
\begin{align*}
    p_{-1}(\fv, \sv) &= 1\\
    p_{0}(\fv, \sv) &= 2\fv\\
    p_{n}(\fv, \sv) &= 2\fv p_{n-1} - (\fv^2 - \sv)p_{n-2}
\end{align*}

\begin{prop}\label{P12}
    For any $0 \le k \le n-1$, 
    \begin{align*} r_n &= p_k r_{n - k - 1} - (\fv^2 - \sv)p_{k-1}r_{n-k-2} \\
    p_n &= p_k p_{n - k - 1} - (\fv^2 - \sv)p_{k-1}p_{n-k-2} \end{align*}
\end{prop}

\begin{proof}
Fix $n$, and induct on $k$. $k=0$ is immediate from the definitions. The proofs for $r_n$ and $p_n$ are essentially identical as they rely only on the three-term recurrences (which are the same) and on the initial conditions $p_{-1}$ and $p_0$. We write out the argument for $r_n$: it amounts to showing that the expression for $k$ equals that for $k+1$:
\begin{align*} 
p_k r_{n - k - 1} - (\fv^2 - \sv)p_{k-1}r_{n-k-2} &= 
p_k (2\fv r_{n-k-2} - (\fv^2 - \sv) r_{n-k-3}) - (\fv^2 - \sv)p_{k-1}r_{n-k-2} \\
&= r_{n-k-2}(2\fv p_k - (\fv^2 - \sv)p_{k-1}) - (\fv^2 - \sv)p_{k}r_{n-k-3} \\
&= p_{k+1}r_{n-k-2} - (\fv^2 - \sv)p_kr_{n-k-3}.
\end{align*}
\end{proof}

In analogy to Section~\ref{sec:un-p-s}, where we studied the roots of the univariate polynomials $\pn_n$, we now need some understanding of where each $r_n$ (which is bivariate, in variables $\fv,\sv$) is zero. 

Notice that $r_i(0, 0) = 0$ for all $i \ge 2$. Since this is a common zero for all $r_i$, $i\geq 2$, we call this the trivial zero.

\begin{lemma}\label{L14} For $i\geq 2$
the only zero of $r_i$ on the curves $\fv = 0$ and $\fv^2 = \sv$ is the trivial zero.
\end{lemma}
\begin{proof}
First, for $\fv = 0$ the recursion takes the form $r_n(0, \sv) = \sv r_{n-2}$, and so:
    \begin{align*}
        r_n(0, \sv) = \begin{cases}
            \gamma \sv^{\lfloor n / 2 \rfloor} &\text{if } n \equiv 0 \mod 2\\
            \sv^{\lfloor n / 2 \rfloor} &\text{if } n \equiv 1 \mod 2
        \end{cases}
    \end{align*}
This implies that if $r_n(0, \sv) = 0$, then $\sv = 0$. (Notice that this also forces $n \ge 2$.)

Second, for $\fv^2 = \sv$: here the recursion takes the form $r_n(\fv, \fv^2) = 2\fv r_{n-1}(\fv, \fv^2)$, and therefore we have $r_n(\fv, \fv^2) = \gamma(2\fv)^{n-1}$. Therefore if $r_n(\fv, \fv^2) = 0$ for $n\geq 2$ then $\fv = 0$.
\end{proof}

\subsection{Common Zeros} A new phenomenon that we encounter, unlike in Section~\ref{sec:un}, is that we need to identify \emph{common zeros} of $r_i,r_j$ for $i\neq j$. 
First we make the following observation.

\begin{lemma}\label{L15} For no $i$ is there a nontrivial zero shared by $r_i$ and $r_{i+1}$, or by $r_i$ and $r_{i + 2}$. 

\end{lemma}
\begin{proof} Pick the smallest such $i$ for which either claim fails, and let $(\fv_0, \sv_0)$ be a nontrivial zero. We know from Lemma \ref{L14} that $\fv_0 \neq 0$ and $\fv_0^2 \neq \sv_0$. If the claim fails because 
$r_i=r_{i+2}=0$, then writing $r_{i+2} = 2\fv_0r_{i+1} - (\fv_0^2 - \sv_0)r_{i}$, we see that necessarily also $r_{i+1}=0$. Then 
    \begin{align*}
        r_{i+1} &= 2\fv_0r_{i} - (\fv_0^2 - \sv_0)r_{i-1}\\
        0 &= (\fv_0^2 - \sv_0)r_{i-1}
    \end{align*}
    Since $\fv_0^2 \neq \sv_0$, it follows that $r_{i - 1} = 0$, which contradicts the minimality of $i$.
\end{proof}

The structure of pairwise-common roots in the $(\fv,\sv)$ plane is complex and has two especially interesting regions: the 
line $\fv = \gamma / 2$ and the parabola $\sv=\fv^2-\gamma^2$. The remainder of our analysis relies on common roots within the first of these regions. 

\subsection{Restricting to the line $\fv = \gamma / 2$}
On this line we have the recurrence $r_n(\gamma / 2, \sv) = \gamma r_{n-1} - (\gamma^2/4 - \sv)r_{n-2}$. Since $p_0(\gamma / 2, \sv) = \gamma$, the initial conditions for the $p$ polynomials and $r$ polynomials are identical evaluated on the line:
\begin{align*}
    p_{-1}(\gamma/2, \sv) = r_0(\gamma/2, \sv)\\
    p_0(\gamma/2, \sv) = r_1(\gamma/2, \sv)
\end{align*}
Also since we have $r_n(\gamma/2, \sv) = \gamma r_{n-1} - (\gamma^2 / 4 - \sv)r_{n-2}$ and $p_n(\gamma/2, \sv) = \gamma p_{n-1} - (\gamma^2 / 4 - \sv)p_{n-2}$, then it must be true that $r_n(\gamma/2, \sv)= p_{n-1}(\gamma/2, \sv)$. Now consider the univariate polynomials defined by the following recursion:
\begin{align*}
    \rtl_0(\sv) &= 0\\
    \rtl_1(\sv) &= 1\\
    \rtl_n(\sv) &= \gamma \rtl_{n-1} - (\gamma^2/4 - \sv)\rtl_{n-2}
\end{align*}
Notice that $r_{n-1}(\gamma/2, \sv) = p_{n-2}(\gamma/2, \sv) = \rtl_n(\sv)$, thus we will turn our attention to the zeros of the $\rtl$ polynomials. From Proposition~\ref{P12} we have the following corollary:
\begin{corollary}\label{C21}
For all $2 \le k \le n-1$,  $\rtl_n = \rtl_{k} \rtl_{n-k+1} - (\gamma^2/4 - \sv)\rtl_{k-1}\rtl_{n-k}$.
\end{corollary}

We can now prove the following, which has no analogue in the uniform-priors case but is key to the general-priors case. 
\begin{theorem}\label{T422}
    $\gcd(\rtl_i, \rtl_j) = \rtl_{\gcd(i,j)}$.
\end{theorem} (Note, this is $\gcd$ in the ring $\rset[\sv]$.) 
\begin{proof}
The theorem will follow from showing that:
\be
    \text{If $j > i$ then $\gcd(\rtl_i, \rtl_j) = \gcd(\rtl_i, \rtl_{j- i}).$} \label{H1} \ee

To show~\eqref{H1}: In Corollary \ref{C21}, since $j > i$, we can substitute $n=j$ and $k = i + 1$ to obtain:
    \begin{align*}
        \rtl_j = \rtl_{i + 1} \rtl_{j - i} - (\gamma^2/4 - \sv)\rtl_{i}\rtl_{j - i - 1}
    \end{align*}
This shows $\gcd(\rtl_i, \rtl_{j - i}) \mid \rtl_j$, and therefore $\gcd(\rtl_i, \rtl_{j - i}) \mid \gcd(\rtl_i, \rtl_j)$. 

It also shows that $\gcd(\rtl_i, \rtl_j) \mid \rtl_{i+1}\rtl_{j - i}$. From Lemma \ref{L15}, $\rtl_i$ and $\rtl_{i + 1}$ are relatively prime, it follows that $\gcd(\rtl_i, \rtl_j) \mid \rtl_{j - i}$.
Consequently $\gcd(\rtl_i, \rtl_j) \mid \gcd(\rtl_i, \rtl_{j-i})$.
\end{proof}

\subsection{Simple roots of $\rtl_n$ polynomials}

We aim to show here that the roots of $\rtl_n$ are simple and real. We start with some useful properties. Recall that $\gamma \in (0,1]$.

\begin{lemma} \label{lsn} For every $n\geq 1$,
\begin{enumerate}
    \item The leading coefficient of $\rtl_n$ is positive. \label{lsn0}
    \item $\rtl_n(0) = n(\frac{\gamma}{2})^{n-1} >0$. \label{lsn1}
    \item The degree of $\rtl_n$ is $\lfloor \frac{n-1}{2} \rfloor$. \label{lsn2}
\end{enumerate}
\end{lemma}

\begin{proof} Clearly, these statements are true of $\rtl_1$ and $\rtl_2$. Now suppose $n>2$. For Part~\ref{lsn0}, observe that the leading coefficient of $\gamma \rtl_{n-1}$ is positive by the inductive hypothesis, and the same is true of $\sv \rtl_{n-2} - (\gamma^2/4) \rtl_{n-2}$. Thus the leading coefficient of $\rtl_n$ is positive. 
For Part~\ref{lsn1}, observe:
\begin{align*}
\rtl_n(0) &= \gamma \rtl_{n-1}(0) - \frac{\gamma^2}{4} \rtl_{n-2}(0) = \gamma (n-1) (\frac{\gamma}{2})^{n-2} - \frac{\gamma^2}{4} (n-2) (\frac{\gamma}{2})^{n-3}
= n(\frac{\gamma}{2})^{n-1}
\end{align*}
For Part~\ref{lsn2}, first suppose $n$ is odd. Then $\rtl_{n-1},\rtl_{n-2}$ have the same degree and so $\rtl_n$ has the same degree as $\sv \rtl_{n-2}$, which is $(n-1)/2$. If $n$ is even, then $\rtl_{n-1}$ has degree $(n-2)/2$ and $\rtl_{n-2}$ has degree $(n-4)/2$. Since the leading coefficients of $\gamma \rtl_{n-1}$ and 
$(\sv-\gamma^2/4)\rtl_{n-2}$ have the same sign by the inductive hypothesis, the degree of $\rtl_n$ is $(n-2)/2$.
\end{proof}

With the aid of Lemma~\ref{lsn} we exhibit an interlacing property of the polynomials $\rtl_n$.

\begin{lemma} \label{lem:int}
For $k\geq 3$, $\rtl_k$ has real roots $\beta_{k,1},\dots,\beta_{k,\lfloor (k-1)/2 \rfloor}$, satisfying the following:
(a) For $n>1$, $\beta_{2n+1,1} < \beta_{2n,1} < \cdots < \beta_{2n,n-1} < \beta_{2n+1,n}<0$. 
(b) For $n\geq 1$, $-\infty < \beta_{2n+1,1} < \beta_{2n+2,1} < \cdots < \beta_{2n+1,n} < \beta_{2n+2,n} < 0$. 

In particular, each $\rtl_k$ has only simple roots.
\end{lemma}

\begin{proof}It is easy to check that $\beta_{3,1}=-\frac{3}{4},\beta_{4,1}=-\frac14$, so $\beta_{3,1}<\beta_{4,1}$, and they are both contained in the interval $(-\infty,0)$. 

We proceed by induction for all $n>1$, treating (a), (b) separately.

(a) Observe that for $i\in  [n-1]$, 
\begin{align*}
\rtl_{2n+1}(\beta_{2n,i}) = -\left(\frac{\gamma^2}{4} - \beta_{2n,i}\right) \rtl_{2n-1}(\beta_{2n,i}).
\end{align*}

By the inductive hypothesis, $\beta_{2n,i}<0$, so $\frac{\gamma^2}{4}-\beta_{2n,i}$ is positive, and for convenience we will denote it by $c_i$. Now note as we range over all $i$, the sign of $\rtl_{2n-1}(\beta_{2n,i})$ changes every time we increment $i$ because $\rtl_{2n-1}$ interlaces $\rtl_{2n}$ by the inductive hypothesis. By the Intermediate Value Theorem, we have found $n-2$ roots in the intervals $(\beta_{2n,i},\beta_{2n,i+1})$ for $i=1,\dots,n-2$. We have two more roots to account for. Note $$\sign(\rtl_{2n+1}(\beta_{2n,1})) = - \sign(\rtl_{2n-1}(\beta_{2n,1})) = \sign\left(\lim_{v\to-\infty} \rtl_{2n-1}(v)\right) = -\sign\left(\lim_{v\to-\infty} \rtl_{2n+1}(v)\right).$$ The first equality holds by the recurrence relation; the second equality holds because $\beta_{2n-1,1}<\beta_{2n,1}$, and the third equality holds because the degrees of $\rtl_{2n-1},\rtl_{2n+1}$ are different by 1. Thus, $\rtl_{2n+1}$ has an odd number of roots, and therefore one root, in the interval $(-\infty,\beta_{2n,1})$. Finally, observe that $\rtl_{2n-1}(\beta_{2n,n-1})>0$ because $\beta_{2n-1,n-1} < \beta_{2n,n-1}$ and $\rtl_{2n-1}$ has positive leading coefficient by the previous lemma. Thus, $\rtl_{2n+1}(\beta_{2n,n-1}) < 0$. Since $\rtl_{2n+1}(0)>0$, $\rtl_{2n+1}$ has a root in the interval $(\beta_{2n,n-1},0)$. Thus we've accounted for all $n$ roots of $\rtl_{2n+1}$, and shown that they are all negative and interlace the roots of $s_{2n}$.

(b) 
We now show that the roots of $\rtl_{2n+2}$ interlace those of $\rtl_{2n+1}$ and $0$. First, observe that $\rtl_{2n+2}(\beta_{2n+1,i}) = -c_i \rtl_{2n}(\beta_{2n+1,i})$ for $i=1,\dots,n$, where we have made the obvious change of definition for $c_i>0$. Since $\rtl_{2n}(\beta_{2n+1,i})$ changes sign every time we increment $i$, by the Intermediate Value Theorem, $\rtl_{2n+2}$ has at least one root each in $(\beta_{2n+1,i},\beta_{2n+1,i+1})$ for $i=1,\dots,n-1$. Finally, we can see that $\rtl_{2n}(\beta_{2n+1,n})>0$ because $\beta_{2n,n-1}<\beta_{2n+1,n}$ and $\rtl_{2n}$ has positive leading coefficient. Thus, $\rtl_{2n+2}(\beta_{2n+1,n})<0$, so $\rtl_{2n+2}$ has a root in the interval $(\beta_{2n+1,n},0)$. We have now accounted for all $n$ roots of $\rtl_{2n+2}$. 
\end{proof}

\subsection{Decomposition of the polynomials $\rtl_n$} Theorem~\ref{T422} and Lemma~\ref{lem:int} actually imply that the polynomials $\rtl_n$ have considerably more structure than already revealed. This will be essential to our results. 
\begin{lemma} There are real polynomials $h_n$ such that $h_n$ has only simple real roots, $\gcd(h_n,h_m)=1$ for $n\neq m$, and for every $n$,
\begin{equation}
\rtl_n(\sv) = \prod_{d\mid n} h_d(\sv).
\label{hprods} \end{equation}
\end{lemma}
\begin{proof}
We prove this by induction on $n$. For the $\gcd$ claim, while treating $n$ we address only $m$ s.t.\ $n\nmid m$. 

Since $\rtl_1=1$, the factorization and real-roots claims hold for $n$ prime, with $h_n=\rtl_n$; the $\gcd$ claim follows from Theorem~\ref{T422}. 

For $n$ composite, we know
from Theorem~\ref{T422} that $\rtl_n$ shares the roots of every $\rtl_d, \, d \mid n$;
by the inductive hypothesis this is equivalent to saying that $\rtl_n$ is divisible by $\prod_{d\mid n, d<n} h_d(\sv)$. 
Since $\rtl_n$ has only simple roots, any remaining factors of $\rtl_n$ cannot be shared with any $h_d$ for $d \mid n, d<n$.
Set $h_n=\rtl_n / \left(\prod_{d\mid n, d<n} h_d(\sv)\right)$. Again, since $\rtl_n$ has only simple roots, $h_n$ is relatively prime to every $h_d, d \mid n$. It remains to show that $h_n$ is relatively prime to $h_m$ for $n \nmid m$. Since $\gcd(\rtl_n,\rtl_m)=\rtl_{\gcd(n,m)}$, which does not include any factors of $h_m$, we know that $h_m$ is relatively prime to $\rtl_n$, and therefore to $h_n$.
\end{proof}

We will refer to the $h_n$ as ``atomic polynomials'' to recognize that they are what compose the $\rtl_n$ polynomials. Next we wish to work out their degrees, which we denote $f(n) = \deg h_n$.
\begin{lemma} $f(1)=0$, $f(2)=0$, and for $n>2$, 
$f(n)=\frac n2 \prod_{q \text{ prime, } q \mid n} (1-1/q)>0$. \end{lemma}
\begin{proof} From~\eqref{hprods}, $\deg \rtl_n=\sum_{d\mid n} f(d).$ Next perform M\"obius inversion in the division lattice to obtain an expression for $f(n)$ in terms of $F(d) \coloneq \deg \rtl_d =  \lfloor (d-1)/2 \rfloor$: that is, 
for $\mu$ the M\"obius function of the division lattice, $f(n)= \sum_{d\mid n} F(d) \mu(n/d)$. Letting the prime factorization of $n$ be $n=q_1^{\beta_1}\cdots q_k^{\beta_k}$ with $q_i<q_{i+1}$, this expression simplifies to $f(n) = \sum_{S\subseteq[k]} (-1)^{|S|} F(n/q^S)$ where $q^S \coloneq \prod_{i\in S}q_i$. Now consider three cases.

First, suppose $n$ is odd. Then $\lfloor (n-1)/2 \rfloor = (n-1)/2$, and $n/q^S$ is odd for any $S$. Observe
\begin{align*}
f(n) &= \frac12 \sum_{S\subseteq[k]} (-1)^{|S|} \left(\frac{n}{q^S}-1\right) = \frac{n}{2} \sum_{S\subseteq[k]} (-1)^{|S|} \frac{1}{q^S} = \frac{n}{2} \prod_{i=1}^k \left(1 - \frac{1}{q_i} \right).
\end{align*}
The second equality follows because $S$ has as many even-sized as odd-sized subsets. 

Second, suppose that $4\mid n$. Now
$n/q^S$ is even for any $S$ because $q^S$ contains at most one factor of $2$.
For even $n$, $\lfloor (n-1)/2 \rfloor = (n-2)/2$. The argument now follows the pattern for $n$ odd.

Third, suppose that $n=2m$, $m>1$ odd. Now $q_1=2,\beta_1=1$.
So for $S \subset [k]$ if $1\notin S$ then $F(n/q^S)=\frac{n}{q^S} - 2$, and if
$1\in S$ then $F(n/q^S)=\frac{n}{q^S} - 1$.

\begin{align*}
f(n) &= \frac12 \sum_{1\notin S} (-1)^{|S|}\left(\frac{n}{q^S} - 2\right) + \frac12 \sum_{1\in S} (-1)^{|S|}\left(\frac{n}{q^S}-1 \right) \\
&= \frac n2 \sum_{1\notin S} (-1)^{|S|} \frac{1}{q^S} - \frac n2 \sum_{1\notin S} (-1)^{|S|} \frac{1}{2q^S} = \frac{n}{4} \prod_{i=2}^k \left(1-\frac{1}{q_i}\right) = \frac{n}{2} \prod_{i=1}^k \left(1-\frac{1}{q_i}\right).
\end{align*}
\end{proof}

For every $n>2$, therefore, $h_n$ possesses a nonempty set of simple roots, called the atomic roots of $h_n$ or $\rtl_n$; these are roots of $\rtl_m$ if and only if $n \mid m$. See Fig.~\ref{fig:gcd}.

\begin{figure}[!h]
\begin{center}
\includegraphics[scale=0.35]{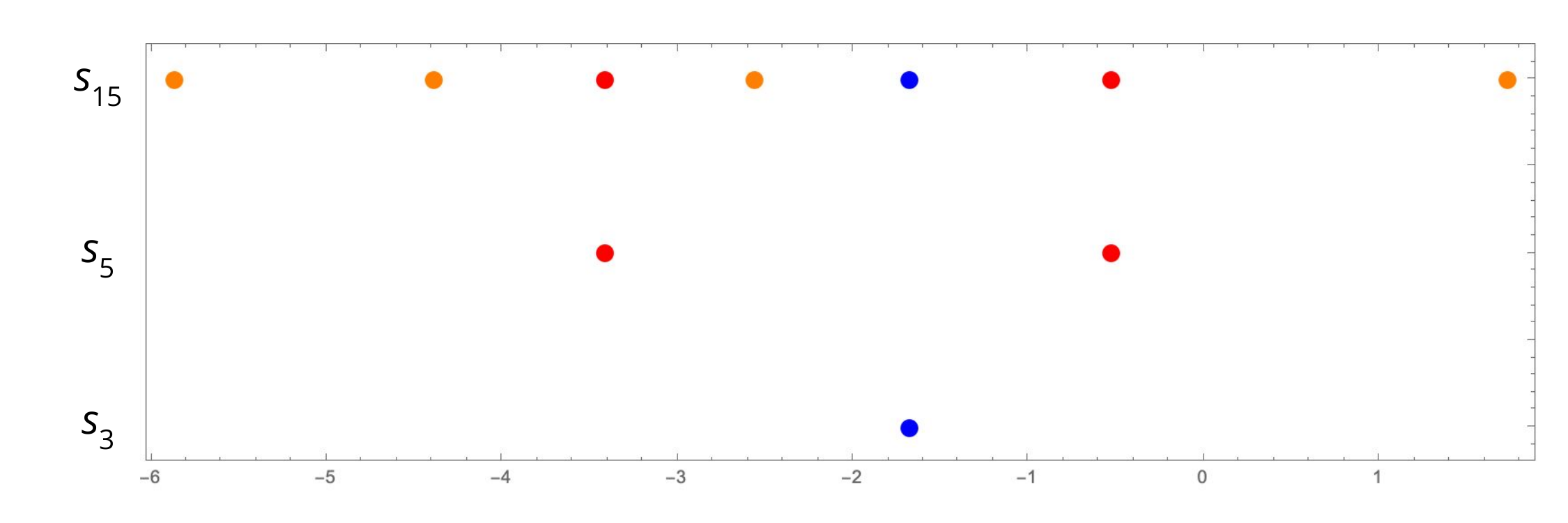}
\caption{The roots of $s_3,s_5$, and $s_{15}$. The blue, red and orange dots represent the roots of the $h_3$, $h_5$ and $h_{15}$ atomic polynomials respectively. For better spacing, the transformation $x\mapsto \log(-x)$ was applied to the horizontal axis.} \label{fig:gcd} 
\end{center}\end{figure}

\subsection{Identifiability: the Jacobian at perturbed common roots}
In this section we finally show how to generalize the method of Sec.~\ref{uni-method}, by leveraging roots shared between pairs of $\rtl_n$'s. Unlike in the basic method introduced in Sec.~\ref{uni-method}, it will not work to evaluate the Jacobian of the mapping exactly at a point specified by these common roots; instead, it will be necessary to slightly perturb the evaluation point.

For any pair $r_i$ and $r_{2i + 1}$, let $J_i$ be the Jacobian of the map $(\fv, \sv) \mapsto (r_i(\fv, \sv), r_{2i + 1}(\fv, \sv))$. Since we have $r_{2i + 1} = p_{i - 1}r_{i + 1} - (\fv^2 - \sv)p_{i-2}r_i$ we can write the partial derivatives of $r_{2i + 1}$ as follows:
\begin{align}
    \frac{\partial r_{2i + 1}}{\partial \fv} &= \frac{\partial p_{i - 1}}{\partial \fv}r_{i + 1} + \frac{\partial r_{i + 1}}{\partial \fv}p_{i - 1} - (\fv^2 - \sv)\left( \frac{\partial p_{i - 2}}{\partial \fv}r_{i} + \frac{\partial r_{i}}{\partial \fv}p_{i - 2} \right) - 2\fv p_{i-2}r_i\\
    \frac{\partial r_{2i + 1}}{\partial \sv} &= \frac{\partial p_{i - 1}}{\partial \sv}r_{i + 1} + \frac{\partial r_{i + 1}}{\partial \sv}p_{i - 1} - (\fv^2 - \sv)\left( \frac{\partial p_{i - 2}}{\partial \sv}r_{i} + \frac{\partial r_{i}}{\partial \sv}p_{i - 2} \right) + p_{i-2}r_i
\end{align}
We define auxiliary polynomials $F_i$ and $G_i$:
\begin{align}
    F_i &= \frac{\partial r_i}{\partial \fv} \frac{\partial r_{i+1}}{\partial \sv} - \frac{\partial r_i}{\partial \sv} \frac{\partial r_{i+1}}{\partial \fv}  \\
    G_i &= p_{i-2}\left(\frac{\partial r_i}{\partial \fv} + 2\fv\frac{\partial r_i}{\partial \sv}\right) - (\fv^2 - \sv)\left(\frac{\partial r_i}{\partial \fv} \frac{\partial p_{i-2}}{\partial \sv} - \frac{\partial r_i}{\partial \sv} \frac{\partial p_{i-2}}{\partial \fv} \right)
\end{align}
Using these expressions above, we have that the determinant of $J_i$ simplifies.
\begin{align}\label{NEWJ}
    \det(J_i) &= \frac{\partial r_i}{\partial \fv} \frac{\partial r_{2i + 1}}{\partial \sv} - \frac{\partial r_i}{\partial \sv} \frac{\partial r_{2i +1}}{\partial \fv} = r_{i + 1}\left( \frac{\partial r_i}{\partial \fv} \frac{\partial p_{i-1}}{\partial \sv} - \frac{\partial r_i}{\partial \sv} \frac{\partial p_{i-1}}{\partial \fv} \right) + p_{i-1}F_i + r_iG_i
\end{align}

\begin{lemma}\label{L02}
    On the line $\fv = \gamma / 2$, $\det(J_i) = -2\rtl_{i+2}\rtl_{i}\frac{\partial \rtl_{i+1}}{\partial \sv} + \rtl_{i+1}(F_i + G_i)$
\end{lemma}
\begin{proof}
    Pick any point on the line $v = (\gamma/2, v_0)$. Notice that we can rewrite $r_i$ and $p_{i - 1}$ as follows:
    \begin{align*}
        p_{i -1} = 2\fv p_{i - 2} - (\fv^2 - \sv)p_{i-3}\\
        r_i = \gamma p_{i-2} - (\fv^2 - \sv)p_{i-3}
    \end{align*}
    Writing the partials of both equations we see:
    \begin{align*}
        \frac{\partial r_i}{\partial \sv} &= \gamma \frac{\partial p_{i-2}}{\partial \sv} - (\fv^2 - \sv)\frac{\partial p_{i-3}}{\partial \sv} + p_{i-3}\\
        \frac{\partial r_i}{\partial \fv} &= \gamma \frac{\partial p_{i-2}}{\partial \fv} - (\fv^2 - \sv)\frac{\partial p_{i-3}}{\partial \fv} - 2\fv p_{i-3}\\
        \frac{\partial p_{i-1}}{\partial \sv} &= 2\fv \frac{\partial p_{i-2}}{\partial \sv} - (\fv^2 - \sv)\frac{\partial p_{i-3}}{\partial \sv} + p_{i-3}\\
        \frac{\partial p_{i-1}}{\partial \fv} &= 2\fv \frac{\partial p_{i-2}}{\partial \fv} - (\fv^2 - \sv)\frac{\partial p_{i-3}}{\partial \fv} - 2\fv p_{i-3} + 2p_{i-2}
    \end{align*}
    Since we are only interested in the solutions on the line $\fv = \frac{\gamma}{2}$, we can now rewrite the following partials:
    \begin{align*}
        \frac{\partial p_{i-1}}{\partial \sv}(v) &= \frac{\partial r_{i}}{\partial \sv}(v)\\
        \frac{\partial p_{i-1}}{\partial \fv}(v) &= \frac{\partial r_{i}}{\partial \fv}(v) + 2\rtl_{i}(v_0)
    \end{align*}
    Finally we see that:
    \begin{align*}
        \frac{\partial p_{i-1}}{\partial \fv}(v) \frac{\partial r_{i}}{\partial \sv}(v) - \frac{\partial p_{i-1}}{\partial \sv}(v) \frac{\partial r_{i}}{\partial \fv}(v)
        &= \left( \frac{\partial r_{i}}{\partial \fv}(v) + 2\rtl_{i}(v_0)\right)\frac{\partial r_{i}}{\partial \sv}(v) - \frac{\partial r_{i}}{\partial \sv}(v) \frac{\partial r_{i}}{\partial \fv}(v)\\ 
        &= 2\rtl_{i}(v_0)\frac{\partial r_{i}}{\partial \sv}(v) = 2\rtl_{i}(v_0)\frac{\partial \rtl_{i+1}}{\partial \sv}(v_0)
    \end{align*}
    We know on the line $\rtl_{i+1} = r_i = p_{i - 1}$ for all $i$, so plugging this back into Equation \ref{NEWJ}, we get the following expression as desired:
    \begin{align*}
        \det (J_i)(v) = -2\rtl_{i+2}(v_0)\rtl_{i}(v_0)\frac{\partial \rtl_{i+1}}{\partial \sv}(v_0) + \rtl_{i+1}(v_0)(F_i(v) + G_i(v))
    \end{align*}
\end{proof}

Now we can begin proving the main theorem of this section.

\begin{theorem} \label{thm:ind-inf-genl}
The map $(\fv_1, \sv_1, \ldots, \fv_\ell, \sv_\ell) \mapsto (q_2, q_5, \ldots, q_{2n}, q_{4n + 1}, \ldots, q_{2\ell}, q_{4\ell + 1})$ is locally identifiable.
\end{theorem}

\begin{proof}
From Theorem \ref{T422} we know $\rtl_{2n+1} \mid \rtl_{4n + 2}$, let $c_n$ be an atomic root of $\rtl_{2n+1}$ and let $C = \{ c_n \}_{n = 1}^\ell$. For ease of notation, let $R_n = \{j \mid \rtl_{2n}(c_j) = 0\}$ and $|R_n| = \alpha_n$. Let us make the following observations about this set:

\begin{obs}\label{L01}
    For all $j \in R_n$, $j \le n$. Moreover, $\rtl_{4n + 2}(c_j) = 0$ if and only if $j \in R_n$
\end{obs}
\begin{proof}
    Suppose $j > n$, we know $c_j$ is an atomic root of $\rtl_{2j+1}$ and by definition of an atomic root $\rtl_{2n+1}(c_j) \neq 0$, thus if $j \in R_n$ then $j \le n$. 
    
    Clearly if $j \in R_n$ then $\rtl_{4n+2}(c_j) = 0$. If $\rtl_{4n + 2}(c_j) = 0$, since $c_j$ is an atomic root of $\rtl_{2j+1}$, then $\gcd(\rtl_{4n+2}, \rtl_{2j + 1}) = \rtl_{2j+1}$. This implies $2j + 1 \mid 4n + 2$ and thus $2j + 1 \mid 2n + 1$ and it follows $\rtl_{2n + 1}(c_j) = 0$ from Theorem \ref{T422} so $j \in R_n$.
\end{proof}

We know that each root of $\rtl_{2n+1}$ is simple for all $n$, thus $({\partial \rtl_{2n+1}}/{\partial \sv})(c_i) \neq 0$. Furthermore, since $\rtl_{2n+1}(c_n) = 0$ then we know that $\rtl_{2n}(c_n) \neq 0$ and $\rtl_{2n + 2}(c_n) \neq 0$. Lastly, notice that from the above observation, $i \in R_n$ if and only if $\rtl_{4n+2}(c_i) = 0$. Together with Theorem \ref{T422} this implies that there exists some $\delta > 0$ such that for all $n$, we have that $\rtl_{2n}$, $\rtl_{2n + 2}$, $\partial \rtl_{2n+1}/\partial \sv$, and $\rtl_{2i+1}$ and $\rtl_{4i+2}$ for all $i \in \{1, \ldots, \ell\} \setminus R_n$, are all bounded away from zero in the interval $I_n = [c_n - \delta, c_n + \delta]$, by some constant $A$.

Clearly $T = \cup_{n=1}^\ell I_n$ is closed and bounded, then so is the set $\{\gamma /2\} \times T$ and thus each of the functions in the following set attain a maximum over $\{\gamma /2\} \times T$: 
\begin{align*}
\bigcup_{n = 1}^\ell \left\{ |r_{2n}|, |r_{2n + 1}|, |r_{2n - 1}|, |\frac{\partial r_{2n}}{\partial \fv}|, |\frac{\partial r_{2n}}{\partial \sv}|, |\frac{\partial r_{4n+1}}{\partial \fv}|, |\frac{\partial r_{4n+1}}{\partial \sv}|, |F_{2n}|, |G_{2n}|\right\}
\end{align*}
Define $M$ to be the maximum over the maximums of these functions and 1. 

We now pick some small $\epsilon > 0$ to be specified later. Define the set $D_j = \{ n \mid j \in R_n \}$. We will pick our points as follows, for all $1 \le i \le \ell$, we pick $d_i \in I_i$ such that $0 < \rtl_{2k+1}(d_i) < \epsilon_i$ and $0 < \rtl_{4k + 2}(d_i) < \epsilon_i$ for all $k \in D_i$. If $1 \le i < \ell$, we define $\epsilon_{i + 1}$ as follows:
    \begin{align*}
        \epsilon_{i+1} = \min \left( \bigcup_{k \in D_i}\{|\rtl_{2k+1}(d_i)|, |\rtl_{4k + 2}(d_i)|\}\right)
    \end{align*}
Notice that this process results in a set of points $\{d_i\}$ where for all $n$ and $k \in R_n$ if $k \neq n$ then $\rtl_{2n+1}(d_n) < \epsilon_{n} < \rtl_{2n+1}(d_k)$ and $\rtl_{4n + 2}(d_n) < \epsilon_{n} < \rtl_{4n + 2}(d_k)$. Therefore for all $k \in R_n$:
    \begin{align}\label{INE}
        \left| \frac{\rtl_{2n+1}(d_n)}{\rtl_{2n+1}(d_k)} \right| \le 1 && \left| \frac{\rtl_{4n+2}(d_n)}{\rtl_{4n+2}(d_k)} \right| \le 1
    \end{align}

    We now evaluate the Jacobian at the point $(\gamma/2, d_1, \ldots, \gamma/2, d_\ell)$. We scale the rows corresponding to $q_{2n}$ and $q_{4n + 1}$ by the following two non-zero values respectively:
    \begin{align*}
        \rtl_{2n+1}(d_n) \prod_{k \in R_n} \frac{1}{\rtl_{2n+1}(d_k)} && \rtl_{4n+2}(d_n) \prod_{k \in R_n} \frac{1}{\rtl_{4n+2}(d_k)}
    \end{align*}
    We call the resulting matrix $B$, and notice that $B$ is non-singular if and only if the Jacobian evaluated at this point is non-singular. For ease of notation, we will refer to the $i,j$th entry of the matrix $B$ as $b_{i, j}$ and we will split the matrix $B$ into $2 \times 2$ blocks.
    \begin{align}
        B = \begin{pmatrix}
            N_{1,1} & \ldots & N_{1,\ell} \\
            \vdots & & \vdots \\
            N_{\ell, 1} & \ldots & N_{\ell, \ell}
        \end{pmatrix}
    \end{align}
    
    Notice that each $2 \times 2$ block has a similar structure, take any $n, m \in [\ell]$ and we can explicitly write the matrix corresponding to $N_{n, m}$.
    \begin{align*}
    N_{n, m} = 
    \begin{pmatrix}
        \frac{\rtl_{2n+1}(d_n)}{\rtl_{2n+1}(d_m)} \prod_{k \notin R_n} \rtl_{2n+1}(d_k) & 0\\
        0 & \frac{\rtl_{4n+2}(d_n)}{\rtl_{4n+2}(d_m)}\prod_{k \notin R_n} \rtl_{4n+2}(d_k)
    \end{pmatrix}
    \begin{pmatrix}
        \frac{\partial r_{2n}}{\partial \fv}(\frac{\gamma}{2}, d_m) & \frac{\partial r_{2n}}{\partial \sv}(\frac{\gamma}{2}, d_m) \\
        \frac{\partial r_{4n+1}}{\partial \fv}(\frac{\gamma}{2}, d_m)  & \frac{\partial r_{4n+1}}{\partial \sv}(\frac{\gamma}{2}, d_m) 
    \end{pmatrix}
    \end{align*}

    \begin{lemma}\label{L04}
        Suppose $b_{i,j} \in N_{n,m}$. If $m \in R_n$ then $|b_{i, j}| \le M^\ell$, otherwise $|b_{i,j}| \le \epsilon M^{\ell - 1}$
    \end{lemma}
    \begin{proof}
        First we will make the following observation, assume that $i$ and $j$ are odd, since $M \ge 1$ we have that:
        \begin{align}\label{MAG}
            \left| \frac{\partial r_{2n}}{\partial \fv}(d_m) \right| \prod_{k \notin R_n} \left| \rtl_{2n+1}(d_k) \right| < M^\ell
        \end{align}
        If $i$ is even we replace $2n$ with $4n + 1$, and if $j$ is even we replace $\fv$ with $\sv$. Notice that the same argument works in all of those cases. If $m \in R_n$, then we know that Equation \ref{INE} implies that both $|\rtl_{2n+1}(d_n)/\rtl_{2n+1}(d_m)| \le 1$ and $|\rtl_{4n+2}(d_n)/\rtl_{4n+2}(d_m)| \le 1$. Together with Equation \ref{MAG} this implies that $|b_{i,j}| < M^\ell$ as desired.

        Suppose that $m \notin R_n$, we will do the following analysis assuming both $i$ and $j$ are odd, but an identical argument works for either $i$ and $j$ even.
        \begin{align*}
            \frac{\partial r_{2n}}{\partial \fv}(\gamma/2, d_m) \cdot \frac{\rtl_{2n+1}(d_n)}{\rtl_{2n+1}(d_m)} \prod_{k \notin R_n} \rtl_{2n+1}(d_k) = \frac{\partial r_{2n}}{\partial \fv}(\gamma/2,d_m) \cdot \rtl_{2n+1}(d_n) \prod_{k \notin R_n \cup \{m\}} \rtl_{2n+1}(d_k)
        \end{align*}
        Notice $|\rtl_{2n+1}(d_n)| < \epsilon_n < \epsilon$ and all other terms are bounded above in magnitude by $M$, since there are at most $\ell - 1$ of those terms and $M > 1$ we have that $|b_{i,j}| < \epsilon M^{\ell - 1}$ as desired.
    \end{proof}

    The intuition for the rest of the proof is as follows. We know that for a block lower-triangular matrix, the determinant is the product of its diagonal blocks. This is because any permutation $\pi$ which selects an element of the lower-left block must also select an element from the zero block and so the term associated to $\pi$ contributes nothing to the determinant. 
    
    Similar to this, we will show that the product of the determinant of the diagonal blocks, $N_{n,n}$, contains a term which is not dependent on $\epsilon$. In the following lemma we will show that all $\pi$ which pick elements of off-diagonal blocks scale with $\epsilon$. This will imply for sufficiently small $\epsilon$ this matrix must be non-singular.

    Let $H \subset S_{2\ell}$ be the subgroup generated by the $\ell$ transpositions $(2n-1, 2n)$ for $n \in [\ell]$.
    \begin{lemma}\label{L05}
        For all $\pi \in S_{2\ell} \setminus H$:
        \begin{align*}
            \left|\sign(\pi) \prod_{i = 1}^\ell b_{i, \pi(i)} \right| < \epsilon M^{\ell^2 - 1}
        \end{align*}
    \end{lemma}
    \begin{proof}
         Pick any  $\pi \in S_{2\ell}$, notice that if for all $i \in \{1, \ldots, \ell\}$, $\pi(2i) \le 2i$ and $\pi(2i - 1) \le 2i$ then $\pi \in H$. Thus since we pick $\pi \in S_{2\ell} \setminus H$, then for some $i$ either $\pi(2i) > 2i$ or $\pi(2i - 1) > 2i$, in both cases we have that for some $j$, $b_{j, \pi(j)}$ lies in $N_{n, m}$ for $m > n$.

         From Observation \ref{L01}, we noted that if $m > n$, then $m \notin R_n$ and thus from Lemma \ref{L04} $|b_{j, \pi(j)}| < \epsilon M^{\ell - 1}$ and for all other $k \in \{1, \ldots, \ell\} \setminus \{j\}$, $|b_{k, \pi(k)}| < M^\ell$. From this, the inequality follows trivially.
    \end{proof}

    Now we turn our attention to the portion of the determinant contributed by all the permutations in $H$.
    \begin{align}
        \sum_{\pi \in H} \text{sign}(\pi) \prod_{i = 1}^\ell b_{i, \pi(i)} = \prod_{n = 1}^\ell \det(N_{n, n}) = \prod_{n = 1}^\ell \det(J_{2n}(d_n)) \left( 
        \prod_{k \notin R_n} \rtl_{2n+1}(d_k)\cdot \rtl_{4n+2}(d_k)
        \right)
    \end{align}
    Notice that using the definitions for $M$ and $A$, we can bound the magnitude of the following product from above and below.
    \begin{align}\label{E10}
        A^{2\ell^2 - 2(\alpha_1 + \cdots + \alpha_\ell)} \le \left| 
        \prod_{n = 1}^\ell \prod_{k \notin R_n} \rtl_{2n+1}(d_k)\cdot \rtl_{4n+2}(d_k)
        \right|
        \le M^{2\ell^2}
    \end{align}
    \begin{lemma}\label{L06}
        \begin{align*}
            \left| \prod_{n = 1}^\ell \det(J_{2n}(d_n)) \right| \ge 2^\ell A^{3\ell} - \epsilon 2^\ell(2^\ell - 1) M^{3\ell - 2}
        \end{align*}
    \end{lemma}
    \begin{proof}
        From Lemma \ref{L02}:
        \begin{align*}
            \prod_{n = 1}^\ell \det(J_{2n}(d_n)) = \prod_{n = 1}^\ell (-2\rtl_{2n+2}(d_n)\rtl_{2n}(d_n)\frac{\partial \rtl_{2n+1}}{\partial \sv}(d_n) + \rtl_{2n+1}(d_n)(F_{2n}(\gamma/2, d_n) + G_{2n}(\gamma/2, d_n)) )
        \end{align*}
        We will first bound the magnitude of the term of this product that results from picking the left term of each factor. From the definition of $A$ we get the following bound.
        \begin{align*}
            \left| \prod_{n = 1}^\ell -2\rtl_{2n+2}(d_n)\rtl_{2n}(d_n)\frac{\partial \rtl_{2n+1}}{\partial \sv}(d_n) \right| \ge 2^\ell A^{3\ell}
        \end{align*}
        All of the rest of the terms must include $\rtl_{2n+1}(d_n)$ for some $n$, and we know that $|\rtl_{2n+1}(d_n)| < \epsilon$ for all $n$. Using this fact and the definition of $M$, we observe the following inequalities.
        \begin{align*}
            |\rtl_{2n+1}(d_n)(F_{2n}(\gamma/2,d_n) + G_{2n}(\gamma/2,d_n)| &< \epsilon (2M)\\
            |-2\rtl_{2n+2}(d_n)\rtl_{2n}(d_n)\frac{\partial \rtl_{2n + 1}}{\partial \sv}(d_n)| &< 2M^3
        \end{align*}
        Clearly $2M^3 > \epsilon(2M)$ so each of the $2^\ell - 1$ other terms in the product are bounded above in magnitude by $\epsilon 2^\ell M^{3\ell - 2}$. Thus the magnitude of the product of the determinants is bounded from below by $2^\ell A^{3\ell} - \epsilon 2^\ell(2^\ell - 1) M^{3\ell - 2}$ as desired.
    \end{proof}

    Using Lemma \ref{L05}, Lemma \ref{L06}, and Equation \ref{E10} we bound the magnitude of the determinant of $B$ from below.
    \begin{align*}
        \left| \det(B) \right| &\ge \left| \sum_{\pi \in H} \sign(\pi) \prod_{i = 1}^\ell b_{i, \pi(i)}\right| - \left|\sum_{\pi \in S_{2\ell} \setminus H} \sign(\pi) \prod_{i = 1}^\ell b_{i, \pi(i)} \right|\\
        &\ge \left( 
        \prod_{k \notin R_n} \rtl_{2n+1}(d_k)\cdot \rtl_{4n+2}(d_k)
        \right)\left( 2^\ell A^{3\ell} - \epsilon 2^\ell(2^\ell - 1) M^{3\ell - 2} \right) - \epsilon (2\ell)! M^{\ell^2 - 1}\\
        &\ge 2^\ell A^{2\ell^2 + 3\ell - 2(\alpha_1 + \cdots + \alpha_\ell)} - \epsilon\left( 2^\ell(2^\ell - 1)M^{2\ell^2 + 3\ell - 2}
        (2\ell)! M^{\ell^2 - 1}
        \right)
    \end{align*}
    Notice that for $\epsilon$ sufficiently small, we have that $|\det(B)| > 0$ and therefore the Jacobian of our desired map is non-singular. This implies that our desired map is locally identifiable, concluding the proof.
    \end{proof}

\section*{Open questions} 

A fundamental issue is whether there is an algorithm for \emph{efficiently} identifying the model from its statistics. Settling this in the positive would be the ideal way also of proving full identifiability. 

A natural question is whether the product in Eqn.~\eqref{munj} can be replaced by other symmetric functions; even more generally, one may consider the  situation in which the effect of the latent variables $U_1,\ldots,U_\ell$ on the observable variables is invariant not under the permutation group $S_\ell$ but under, say, a transitive subgroup of $S_\ell$.

A separate direction to pursue is, under what conditions can the model be identified if it has $n$ observables $X_1,\ldots,X_n$ which are independent, but not necessarily iid, conditional on $\cU$. (This is commonly the assumed setting for RBM's.) That is, Eqn.~\eqref{munj} would be replaced by 
$\mun_S = \prod_{j=1}^\ell \frac{\prod_{i\in S} \alpha_{i,j,0}+ \prod_{i \in S}\alpha_{i,j,1}}{2}$
for each $S\subseteq [n]$. As a positive indication, for the case $\ell=1$, but with an arbitrary finite range for $\cU$, identification was shown~\cite{gordon2021source} to be possible under such conditions, without necessitating any increase in $n$.

A full understanding of this family of problems requires also extending to non-binary $X_i$ and $U_j$. The former is likely straightforward following the lead of~\cite{FanLi22}, but non-binary $U_j$ will demand replacing our two-dimensional space $(\fv,\sv)$ by a higher-dimensional space and, perhaps, generalizing our approach through pairwise-common zeros, to zeros shared by larger assemblies of polynomials.

\section*{Acknowledgments}
We thank Caroline Uhler, Bernd Sturmfels, Yulia Alexandr and Guido Mont\'ufar for helpful discussions.

\bibliography{refs,refs2}

\newpage
\appendix
\section{General condition for applying root and pole information} \label{Jac-cond}
Observe that the process~\eqref{ge-r} is effectively Gaussian elimination on the rows of the matrix $M$, which starts out lower-triangular with $1$'s on the diagonal. Carried further this yields:
\begin{theorem} \label{blowup}
Let $\pn_1,\dotsc,\pn_{\ell}$ be univariate 
rational functions and let $q_i(y) \coloneqq \prod_{j=1}^{\ell} \pn_i(y_j)$ for $i=1,\dotsc,\ell$. Let $\eta_1,\dotsc,\eta_{L}$ be the points which are 
roots or poles of any $\pn_i$. 
Then the mapping $(y_1,\dotsc,y_\ell)\mapsto (q_1(y),\dotsc,q_\ell(y))$ is locally identifiable if and only if the $\ell\times L$ matrix $M(\pn,\eta)$ with $(i,j)$'th entry $M(\pn_i,\eta_j)$, has rank $\ell$ over $\qset$.
\end{theorem}\begin{proof} 
\emph{Only If:} Let $v\in \zset^\ell$ be a linear dependence of the rows, $v \cdot M(\pn,\eta)=0$. Then 
$\prod_{i'=1}^\ell q_{i'}(y)^{v_{i'}}$ factors as 
$\prod_j \prod_{i'=1}^\ell \pn_{i'}(y_j)^{v_{i'}}$. By construction $\prod_{i'=1}^\ell \pn_{i'}(\eta_j)^{v_{i'}}$ is nonzero and finite for every $1\leq j \leq L$. 
Furthermore $\prod_{i'=1}^\ell \pn_{i'}(x)^{v_{i'}}$ is nonzero and finite for all $x \notin \{\eta_1,\ldots,\eta_L\}$ because for such $x$ every $\pn_{i'}(x)$ is nonzero and finite. Thus, $\prod_{i'=1}^\ell \pn_{i'}^{v_{i'}}$ is a rational function without finite roots
or poles, and therefore a nonzero constant. So
$\prod_{i'=1}^\ell q_{i'}(y)^{v_{i'}}$ is a nonzero constant. Consequently, the parameterized variety 
$q(y)=(q_1(y),\ldots,q_{\ell}(y))$ has
codimension at least $1$ in $\cset^\ell$.

\emph{If:} Without loss of generality suppose the submatrix of $M(\pn,\eta)$ in columns $1,\ldots,\ell$ is nonsingular. 
Let $N$ be a matrix with integer entries such that
$N \cdot M(\pn,\eta)=(D \mid B)$ where $D$ is a diagonal matrix with positive integer entries on the diagonal, and
$B$ is any $\ell \times (L-\ell)$ matrix; $\mid$ denotes concatenation. Define the rational functions 
\be r_i = \prod_{i'=1}^\ell \pn_{i'}^{N_{ii'}} \label{rprod} \ee
By construction, for $j\leq \ell$, $M(r_i,\eta_j)=D_{ij}$.
Define $s_i(y) \coloneqq \prod_{j=1}^\ell r_i(y_j)$ for $i=1,\dotsc,\ell$ so that $s_i$ is the product of $r_i$ evaluated at each indeterminate, just as $q_i$ is the product of $\pn_i$ evaluated at each indeterminate. Then 
\[ s_i(y)=\prod_{i'=1}^\ell q_{i'}(y)^{N_{ii'}} \]
Unlike in the proof of Prop.~\ref{prop:identifiability of symmetric polynomials}, it is not sufficient to consider the Jacobian of the mapping 
\[ (y_1,\dotsc,y_{\ell}) \mapsto (s_1,\dotsc,s_{\ell})\]
because this Jacobian is singular if any $D_{ii}>1$. However, we show the mapping is dimension-preserving by examining its expansion in a small neighborhood of $(\eta_1,\ldots,\eta_\ell)$. Observe, as in~\eqref{pspe}, that
\be \frac{\partial^k s_i}{\partial^k y_j}(\eta) = \Paren{\prod_{j'\neq j} r_i(\eta_{j'})} r^{(k)}_i(\eta_j). \label{pspe2} \ee
with $ r^{(k)}_i$ being the $k$'th derivative of $r_i$.
More generally, if $\vec{k}=(k_1,\ldots,k_\ell)$, $|\vec{k}|=\sum k_j$, $ s_i^{(\vec{k})}(y) = 
\frac{\partial^{k_1}}{\partial^{k_1} y_1} \cdots \frac{\partial^{k_\ell}}{\partial^{k_\ell} y_\ell} s_i(y)$,  then 
\[ s_i^{(\vec{k})}(\eta) = \prod_j r_i^{(k_j)}(\eta_j). \]
Let $s=(s_1,\ldots,s_\ell)$. 
In a small neighborhood of $\eta$, $s$ expands in terms of those nonzero partial derivatives $(\vec{k})$ for which $\vec{k}$ is minimal in the standard partial order on the nonnegative quadrant. For each $s_i$ this minimizer is unique, $(0,\ldots,0,D_{ii},0,\ldots,0)$. Thus (applying~\eqref{pspe2}), for small $\ep=(\ep_1,\ldots,\ep_\ell)$,
$s(\eta+\ep)$
  expands as 
\[ \Paren{s_1(\eta),\ldots,s_\ell(\eta)}+\Paren{\Paren{\prod_{j'\neq 1} r_1(\eta_{j'})} r^{(D_{11})}_1(\eta_1)\ep_1^{D_{11}},
\ldots,
\Paren{\prod_{j'\neq \ell} r_\ell(\eta_{j'})} r^{(D_{\ell \ell})}_\ell(\eta_\ell) \ep_\ell^{D_{\ell \ell}}
}\]
This mapping carries $\ep$ in a small open neighborhood of $0$ in $\cset^\ell$, onto an open neighborhood of $s(\eta)$. 
\end{proof}
\end{document}